\def\year{2021}\relax
\documentclass[letterpaper]{article} 
\usepackage{aaai21}  
\usepackage{times}  
\usepackage{helvet} 
\usepackage{courier}  
\usepackage[hyphens]{url}  
\usepackage{graphicx} 
\usepackage{graphicx}  
\usepackage{natbib}  
\usepackage{caption} 
\usepackage{multirow}
\usepackage{booktabs}
\usepackage{threeparttable}
\usepackage{subfigure}
\usepackage{amsmath}
\usepackage{algorithm}
\usepackage{algorithmic}

\usepackage{bbm}
\usepackage{amsthm}
\newtheorem{thm}{Theorem}

\title{Angular Embedding: A New Angular Robust Principal Component Analysis}
\author{
    Shenglan Liu\textsuperscript{\rm 1}
    \thanks{Corressponding author.},
    Yang Yu\textsuperscript{\rm 1}
}
\affiliations{

    \textsuperscript{\rm 1} School of Computer Science and Technology, Faculty of Electronic Information and Electrical Engineering,\\

    Dalian University of Technology, Dalian 116024, Liaoning, China.\\
    \{liusl, yyu\}@mail.dlut.edu.cn

}

\begin{document}
\maketitle

\begin{abstract}
As a widely used method in machine learning, principal component analysis (PCA) shows excellent properties for dimensionality reduction. It is a serious problem that PCA is sensitive to outliers, which has been improved by numerous Robust PCA (RPCA) versions. However, the existing state-of-the-art RPCA approaches cannot easily remove or tolerate outliers by a non-iterative manner. To tackle this issue, this paper proposes Angular Embedding (AE) to formulate a straightforward RPCA approach based on angular density, which is improved for large scale or high-dimensional data. Furthermore, a trimmed AE (TAE) is introduced to deal with data with large scale outliers. Extensive experiments on both synthetic and real-world datasets with vector-level or pixel-level outliers demonstrate that the proposed AE/TAE outperforms the state-of-the-art RPCA based methods.

\end{abstract}

\section{Introduction}

As machine leaning is widely used in many applications, principal component analysis (PCA) \cite{wold1987principal} has already become a remarkable method for dimensionality reduction \cite{vasan2016dimensionality, adiwijaya2018dimensionality}, computer vision \cite{bouwmans2018applications}, etc. However, one of the most important issues of PCA is that the principal components (PCs) are sensitive to outliers \cite{zhao2014robust}, which could not be well addressed under $\ell_2$-norm. RPCA based methods \cite{candes2011robust, zhao2014robust} enhance the robustness of PCA through low-rank decomposition. Besides, iterative subspace learning \cite{roweis1998algorithms, hauberg2014grassmann} is another approach to realize robust PCs for data mining tasks (e.g., data classification \cite{xia2013semi}, clustering \cite{ding2006r} and information retrieval \cite{wang2015semantic}). However, most existing robust methods for PCA are iterative optimization and always free of PCs. The outliers cannot be easily removed or tolerated by non-iterative approaches, which limits the applications on real-world problems (e.g., image analysis in medical science \cite{lazcano2017porting} or FPGA \cite{fernandez2019fpga}). 

To tackle the problems above, in this paper, a straightforward RPCA approach named Angular Embedding (AE), which optimizes PCs with angular density, is proposed with quadratic optimization of cosine value on hypersphere manifold. Based on the cosine measurement, AE enhances the robustness of $\ell_2$-norm based PCA and outperforms the existing methods. As an eigen decomposition based PCA approach, AE is improved to reduce the computational complexity by weighing the dimensionality and the sample size. 

Theoretically, we prove the superiority of quadratic cosine optimization in AE with analyzing the effects of on the determination of PCs and outlier suppression. And in practice, the experimental results of AE on both synthetic and real data show its effectiveness. Furthermore, to address data with large scale outliers, we propose a pre-trimming theory based on cosine measurement and propose trimmed AE (TAE). Base on TAE theory, the experiments on background modeling and shadow removal tasks show its superiority.



\section{Related Work}

In recent years, robust extensions for PCA have attracted increasing attention due to its wide applications. Most existing methods could fall into three categories as follows.

\subsubsection{RPCA with low-rank representation.} Robust PCA \cite{candes2011robust} provides a new robust approach to recover both the low-rank and the sparse components by decomposing the data matrix as ${\bf X} = {\bf L} + {\bf S}$ with Principal Component Pursuit. Based on RPCA, many improved versions have been proposed including the tensor version \cite{lu2019tensor} (require more memory for SVD results) to reduce the error of low-dimensional representation caused by outliers. Methods such as \cite{candes2011robust} require iterative calculation SVD of data with cubic time complexity. \cite{xue2018informed, yi2016fast} utilize gradient descent to avoid frequent maxtrix decompositions. Although the above dimensionality reduction method is robust, it is difficult to fit on large datasets and to design a robust mathematical model for capturing the linear transformation matrix between latent and observed variables. All these reasons limit the applications of the RPCA methods above.

\subsubsection{RPCA for subspace learning.} Subspace learning aims to obtain low-dimensional subspace which spans with robust projection matrix. $\ell_1$-norm based subspace estimation \cite{ke2005robust} is a robust version for standard PCA using alternative convex programming. The determination of robust subspaces with $\ell_1$-norm based methods, however, is computationally expensive. Besides, questionable results in some tasks like clustering would be produced since the global solutions are not rotationally invariant. To this end, $R_1$-PCA \cite{ding2006r} is proposed to improve $\ell_1$-norm PCA. Unfortunately, $R_1$-PCA is also time-consuming, similar to another work named Deterministic High-dimensional Robust PCA \cite{feng2012robust} which updates the weighted covariance matrix with frequent matrix decomposition. Furthermore, based on statistical features such as covariance, Roweis proposed EM PCA \cite{roweis1998algorithms} using EM algorithm to obtain PCs for Gaussian data.

\subsubsection{RPCA with angle measurement.} Recently, angle-based methods \cite{graf2003classification, liu2014scatter, wang2018cosface, wang2017angle} become new approaches in machine learning. The angle metric between samples is less sensitive to outliers and has been applied in many domains (e.g., RNA structures analysis \cite{sargsyan2012geopca}, texture mapping in computer vision \cite{wilson2014spherical}). The non-iterative linear dimensionality reduction method \cite{liu2014scatter}, which utilizes cosine value to obtain robust projection matrix, motivates the angle-based RPCA methods. Angle PCA \cite{wang2017angle} employed cotangent value and iterative matrix decomposition to realize the idea of PCA, which is robust but computationally expensive. Based on EM PCA \cite{roweis1998algorithms} and (trimmed) averages, a more scalable approach in introduced by Grassmann Average (GA) \cite{hauberg2014grassmann}, which is also related to the angle measurement. For GA, the drawback is lower parallelization for calculating large PCs because of iterative matrix multiplications and orthogonalization for PCs. Actually, most RPCA methods is based on the iterative solutions, which are limited by the time-consuming steps (e.g., matrix decomposition in PCA) and unused computational resources for parallelization on CPUs and GPUs, etc.

\section{Angular Robust Principal Component Analysis on Hypersphere Manifold}

$\ell_2$-norm based PCA are sensitive to outliers and can be misled by the Euclidean distance. Many researches on robust PCA pay attention to characterize the error with $\ell_1$-norm or other approaches. The proposed AE, which follows $\ell_2$-norm, turns to determine the principal components (PCs) based on angular density instead of calculating distance. In practice, angular density will not be misled by outliers with large Euclidean distance, especially for the directional outliers (far from the direction of PCs).

\subsection{Angular Density Framework on Hypersphere}

Given a set of zero-mean data ${\bf X} = \{{\bf x}_1, {\bf x}_2, \cdots, {\bf x}_n\} \subset \mathbbm{R}^D$ under the assumption of Gaussian distribution, the angular density can be defined as the number of samples within the unit angle. In theory, AE tends to determined PCs based on angular density. 

The straightforward measurement of angular density is much more complicated than original PCA. Considering the geodesic distance $\mathcal{D}_{ij}$ between any two normalized samples ${\bf u}_i$ and ${\bf u}_j$ on the unit (radius $\mathcal{R}=1$) hypersphere manifold, 

\begin{equation}
    \label{surface-density}
        \mathcal{D}_{ij} = \mathcal{R} \beta_{ij} = \beta_{ij},
\end{equation}

\noindent where $\beta_{ij} = \langle {\bf u}_i, {\bf u}_j \rangle$ is in radians. That is, the angular density between two samples can be quantified by measuring the surface density on a unit hypersphere manifold of codimension one in $D$ dimensions. The input samples can thus be firstly normalized by mapping the original $D$-dimensional zero-mean samples into a unit hypersphere manifold. For $i = 1, 2, \cdots, n$, the unit vector ${\bf u}_i$ corresponding to each sample ${\bf x}_i$ can be obtained by computing 

\begin{equation}
    \label{normalize}
        {\bf u}_i = \frac{{\bf x}_i}{\| {\bf x}_i \|_2}.
\end{equation}

\noindent The normalized none zero mean inputs on $(D-1)$-sphere can be marked as ${\bf U} = \{{\bf u}_1, {\bf u}_2, \cdots, {\bf u}_n\}$. 

\subsubsection{The leading PC.} Given the definition that ${\bf q}$ is the leading PC, which corresponds to the position of zero angle in $D$ dimensions, then each sample ${\bf u}_i$ can be represented as a directional angle $\theta_i \in (-\pi, \pi]$, in radians. To further simplify the calculation, the optimization of angular density are determined by utilizing the sine value of directional variable $\theta_i$ instead of the geodesic distance $\mathcal{D}_{ij}$ on the hypersphere. Then the leading PC can be determined by formulating

\begin{equation}
\label{optimization-sin-leading-PC}
\begin{split}
    {\bf q} 
    &= \arg \min_{\bf q} \sum_{i=1}^{n} \sin ^2 \theta_i \\
    &= \arg \max_{\bf q} \sum_{i=1}^{n} \cos ^2 \theta_i \\
    &= \arg \max_{\bf q} {\bf q}^T {\bf U} {\bf U}^{T} {\bf q} 
\end{split}.
\end{equation}

\subsubsection{Multiple PCs.} Let ${\bf Q} = \{{\bf q}_1, \cdots, {\bf q}_d\}\subset \mathbbm{R}^D$ be the top-$d$ orthogonal PCs. Then the orthogonal projection $\hat {\bf u}_i = {\bf Q}{\bf Q}^T {\bf u}_i$ of the sample ${\bf u}_i$ in $\mathbbm{R}^{D}$ can be represented as $\hat {\bf u}_i = \sum_{j=1}^{d} ({\bf u}_i^T {\bf q}_j) {\bf q}_j / ({\bf q}_j^T {\bf q}_j) $. We define $\theta_i$ to be the angle between the sample ${\bf u}_i$ and its projection $\hat {\bf u}_i$. Each ${\bf q}_j ,j \in \{1,\cdots, d\}$ contributes to $\theta_i$ under $\cos^2 \theta_i = \sum_{j=1}^d \cos^2 \vartheta_{ij}$, where $\vartheta_{i_j} = \langle {\bf u}_i, {\bf q}_j \rangle$ indicates the angle between ${\bf u}_i$ and ${\bf q}_j$ in $D$ dimensions.
Then PCs can be determined by formulating angular density based PCA as 

\begin{equation}
\label{optimization-sin}
\begin{split}
    {\bf Q} &= \arg \max_{\bf Q} \sum_{i=1}^{n} \sum_{j=1}^{d} \cos^2 \vartheta_{ji}\\
    &= \arg \max_{\bf Q} \text{trace} \left( {\bf Q}^T {\bf U} {\bf U}^T {\bf Q} \right)
\end{split}.
\end{equation}

\noindent The minimum reconstruction between sample ${\bf x}_i$ on a $D-1$-sphere and its projection is equivalent to maximize the squared cosine value. The combined AE algorithm can be found in Algorithm. \ref{algorithm-AE}.

\begin{algorithm}
    \caption{\textit{AE Algorithm}}
    \label{algorithm-AE}
    \begin{algorithmic} 
        \FOR{each sample ${\bf x}_i,i=1,\dots,n$}
        \STATE {${\bf u}_i \leftarrow {\bf x}_i / \|{\bf x}_i\|_2$}
        \ENDFOR
        \STATE{${\bf Q} \leftarrow \arg \max\limits_{\bf Q} \text{trace} \left( {\bf Q}^T {\bf U} {\bf U}^T {\bf Q} \right)$}
    \end{algorithmic}
\end{algorithm}

\subsection{Why the Squared Cosine Measurement?}

The measurement of angular density is transformed into cosine value for normalized data on hypersphere in $D$ dimensions. In this section, we demonstrate the robustness of our proposed AE in theory by analyzing the effects of squared cosine value on both the determination of PCs and outlier suppression. 

\subsubsection{$\ell_2$-norm based cosine attention mechanism.} Let the  $d$-dimensional vector $\Theta_i = [\cos \vartheta_{i1}, \cdots, \cos \vartheta_{id}]^T$ be the vector representation of $\theta_i$. Then $\cos \theta_i$ follows the $\ell_2$-norm of $\Theta_i$. In practice, Euclidean distance based PCA is less robust with $\ell_2$-norm, which would be turned into strength in AE. Considering the angle $\theta_i$ between $0$ and $\pi / 2$ in radians, the value of $\cos^2 \theta_i$ tends to reduce more sharply compared with $|\cos \theta_i|$ as $\theta_i$ increases, since they follow $|\cos \theta_i| / \cos^2 \theta_i = 1 / |\cos \theta_i|$. We reformulate Eq. \ref{optimization-sin} as $\psi({\bf u}_i, {\bf Q}) = \sum_i \lambda_i \cos \theta_i$, where $\lambda_i = \cos \theta_i$ is the attention factor. More specifically, the $\ell_2$-norm for $\Theta_i$ achieves the attention-like mechanism, which would reduce the contributions of those samples far from the directions of PCs (closer to $\pi / 2$). Thus, PCs in AE always pays more attention to select the directions corresponding to higher angular density since the samples would produce higher squared cosine values. 

\subsubsection{Outlier suppression via squared consine value.} The cosine function is widely utilized to measure the similarity of vectors \cite{nguyen2010cosine} in high dimensions, where the lower cosine value corresponds to higher similarity. Considering a single PC ${\bf q}$, $\theta_i = \langle {\bf q}, {\bf u}_i \rangle$ will be the angle between a sample ${\bf u}_i$ and the leading component ${\bf q}$. Then the samples corresponding to small $\cos \theta_i$ values would be possible outliers, since the similarity between ${\bf q}$ and ${\bf u}_i$ is lower. For PCA, the computed principal components are more likely to be misled by the samples far from the intrinsic principal components. For AE, by contrast, the influence of outliers can be suppressed by normal samples illustrated as Fig. \ref{outliers-suppression}. 

\begin{figure}[!t]
    \centering
    \includegraphics[width=8cm]{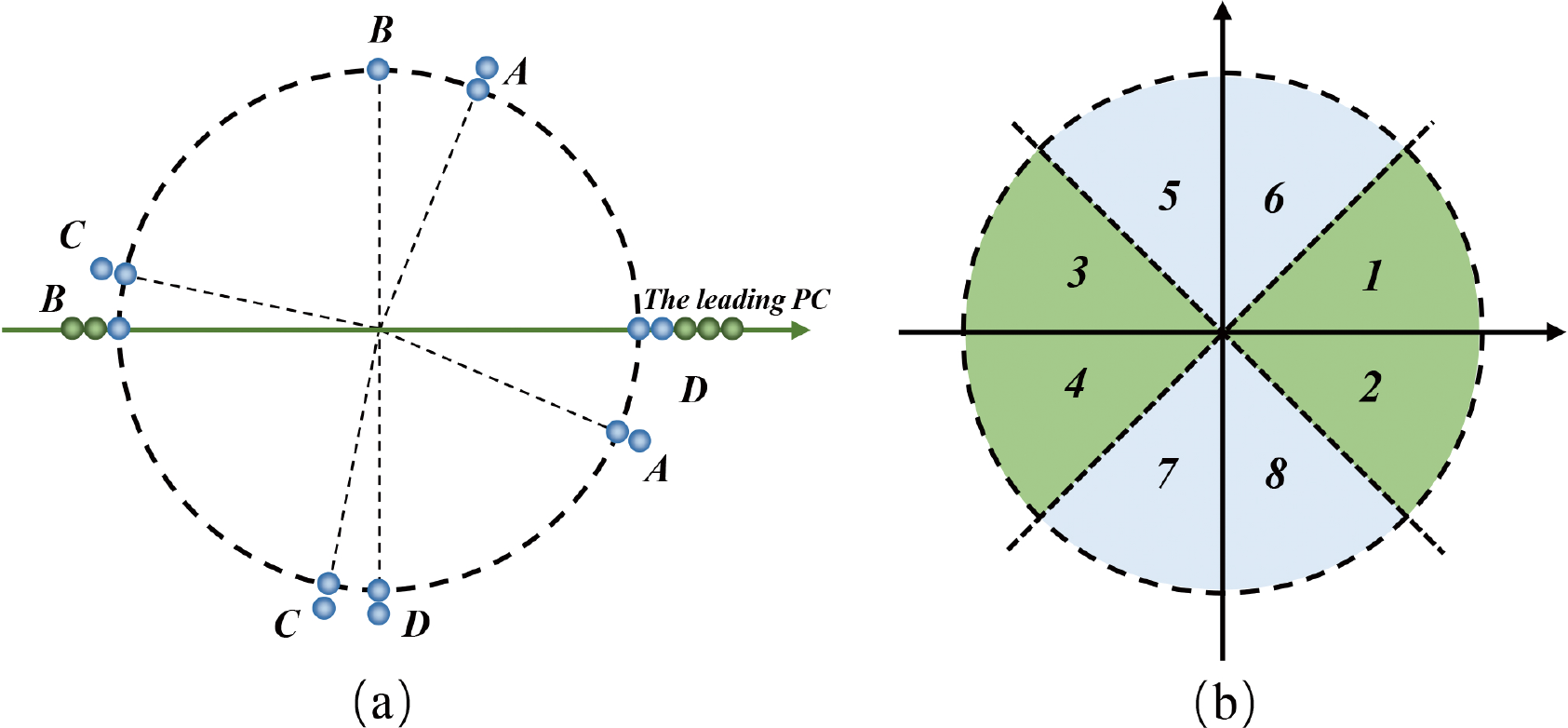}
    \caption{Outlier suppression mechanism under squared cosine optimization in 2 dimensions. All the samples distributed on the circumference, where the radial samples in (a) represent multiple samples at the same localizations. The symbols \textit{A}, \textit{B}, \textit{C} and \textit{D} indicate four orthogonal groups, where the pairwise blue samples in the same class do not contribute to the determination of PCs for AE. Suppose that the green areas (1, 2, 3, and 4) in (b) correspond to higher sample density, and then outliers in blue areas (5, 6, 7, and 8) would be suppressed by normal samples in green area (4, 2, 3, and 1) based on Theorem. \ref{theorem-outlier-suppress}, respectively.}
    \label{outliers-suppression}
\end{figure}

\begin{thm}
    \label{theorem-outlier-suppress}
    The PC ${\bf q}$ with Eq. \ref{optimization-sin-leading-PC} is less influenced by the pairwise vectors ${\bf u}_i$ and ${\bf u}_j$, which are nearly orthogonal, when ${\bf q}$ distributes on the hyperplane $\mathcal{H}$ spanned by ${\bf u}_i$, ${\bf u}_j$ and the origin $O$.
\end{thm}

\begin{proof}
    We mark $\langle {\bf v}, {\bf w} \rangle$ as the angle between vectors ${\bf v}$ and ${\bf w}$ in $D$ dimensions. When the PC ${\bf q}$ distributes on the hyperplane $\mathcal{H}$ spanned by ${\bf u}_i$, ${\bf u}_j$ and the origin $O$ in $D$ dimensions, we consider the following two conditions. If $\langle {\bf u}_i, {\bf u}_j \rangle = \pi / 2$ is satisfied strictly, Eq. \ref{optimization-sin} follows $\varphi({\bf q}, {\bf u}_i, {\bf u}_j) = \sum_{r \in \{i, j\}} \cos ^2 \langle {\bf q}, {\bf u}_r \rangle = \text{const}$. Then, ${\bf u}_i$ and ${\bf u}_j$ contribute nothing to the determination of PC on the hyperplane $\mathcal{H}$. When there is a slight deviation $\xi$ for the angle $\langle {\bf u}_i, {\bf u}_j \rangle$ from $\pi / 2$, that is, $\langle {\bf u}_i, {\bf u}_j \rangle = \pi / 2 \pm \xi$, we suppose $\tilde {\bf u}_i$ as the substitution of ${\bf u}_i$ that satisfies $\langle \tilde {\bf u}_i, {\bf u}_j \rangle = \pi / 2$ and $\langle \tilde {\bf u}_i, {\bf u}_i \rangle = \xi$. And $\varphi({\bf q}, {\bf u}_i, {\bf u}_j) = \sum_{r \in \{i, j\}} \cos ^2 \langle {\bf q}, {\bf u}_r \rangle = \cos ^2 \langle {\bf q}, \tilde {\bf u}_i \rangle + \cos ^2 \langle {\bf q}, {\bf u}_j \rangle + \cos ^2 \langle {\bf q}, {\bf u}_i \rangle - \cos ^2 \langle {\bf q}, \tilde {\bf u}_i \rangle = \text{const} + \cos ^2 \langle {\bf q}, {\bf u}_i \rangle - \cos ^2 \langle {\bf q}, \tilde {\bf u}_i \rangle = \text{const} \pm \sin (\langle {\bf q}, {\bf u}_i \rangle + \langle {\bf q}, \tilde {\bf u}_i \rangle) \sin \xi$ can be obtained with Eq. \ref{optimization-sin}. Thus, the bias for the objective is $|\sin (\langle {\bf q}, {\bf u}_i \rangle + \langle {\bf q}, \tilde {\bf u}_i \rangle) \sin \xi| \leq |\sin \xi|$. In summary, ${\bf u}_i$ and ${\bf u}_j$ make few contributions to PCs on the hyperplane $\mathcal{H}$ under the error $|\sin \xi|$.
\end{proof}

Normal samples that distribute on the area of higher density on the hypersphere are always dominant for the input. Four groups \textit{A}, \textit{B}, \textit{C} and \textit{D} in Fig. \ref{outliers-suppression} (a) with orthogonal unit vectors lie on the circumference. Based on Theorem. \ref{theorem-outlier-suppress}, the orthogonal samples (marked as blue) within the same group make no contributions to the leading PC in Eq. \ref{optimization-sin-leading-PC} so that the remaining samples (marked as green) determine the leading PC. As shown in Fig. \ref{outliers-suppression} (b), the green areas (Area 1, 2, 3, and 4) correspond to higher sample density (normal samples) while the blue area (Area 5, 6, 7, and 8) are of lower density (outliers). If normal samples are covered densely enough, the leading PC can be restricted to the normal area (Area 1, 2, 3, and 5) according to Theorem. \ref{theorem-outlier-suppress}.

\subsection{Trimmed Angular Embedding (TAE)}

When the proportion of normal samples is challenged by large scale outliers in the dataset, the PCs could not be the optimal direction in AE. To remove the misdirection of outliers, the original samples should be trimmed in advance with TAE. The pre-trimming theory of TAE is based on the following items:

\begin{enumerate}
    \item The number of normal samples dominates the total data distribution.
    \item The PCs depend on the directions of large angular density.
    \item Outliers always distribute far from PCs.
\end{enumerate}

\begin{figure}[!t]
    \centering
    \includegraphics[width=8cm]{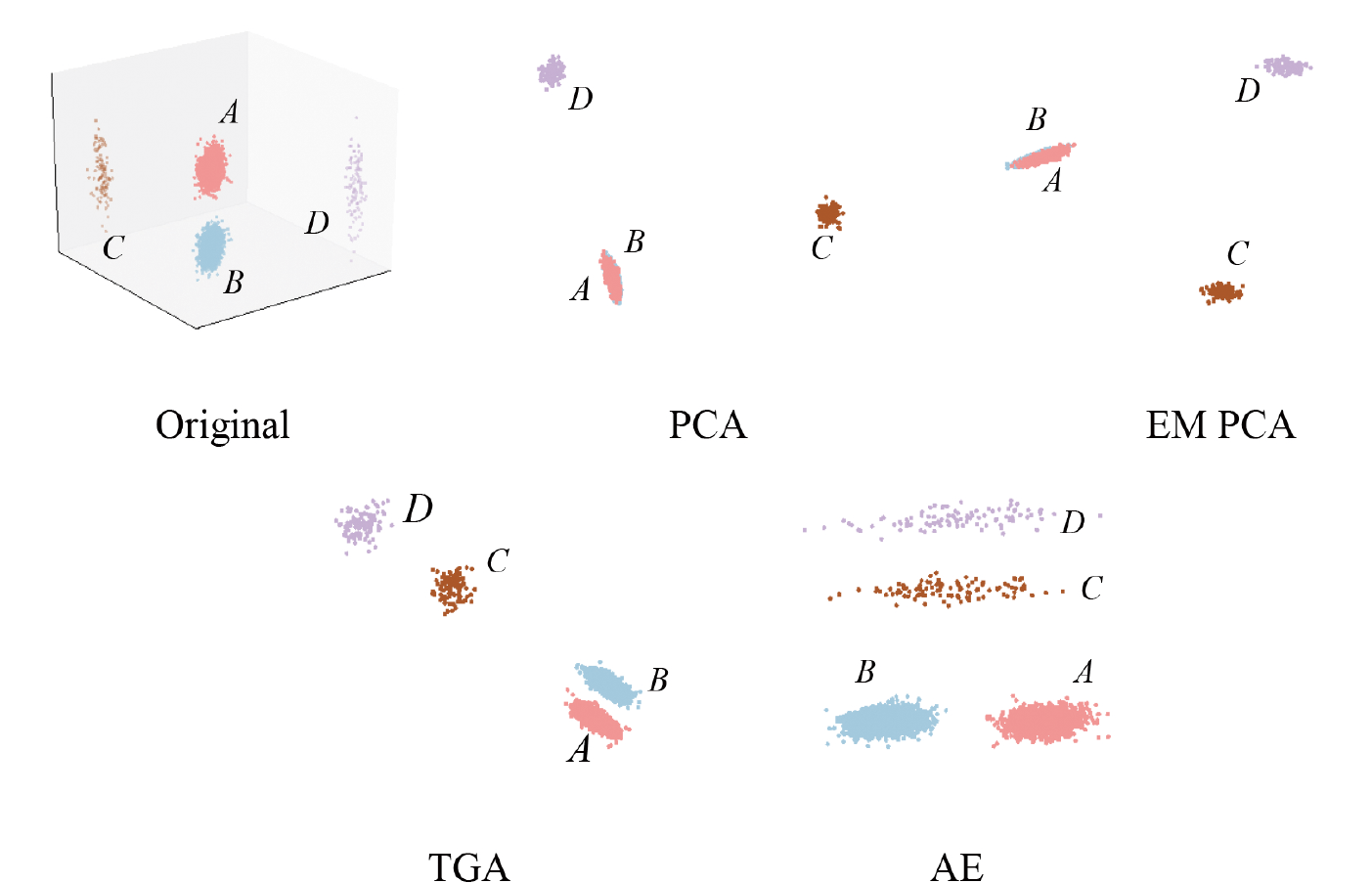}
    \caption{The results of dimensionality reduction on synthetic data. A, B, C and D represent samples of 4 classes.}
    \label{3d-data-dr}
\end{figure}

Similar to AE with the cosine optimization, the trimming principle is under cosine measurement, similarly. For samples ${\bf u}_i$ and ${\bf u}_j$ distributed on the hypersphere in $D$ dimensions, the cosine value can be easily quantified as $\cos \langle {\bf u}_i, {\bf u}_j \rangle = {\bf u}_i^T {\bf u}_j$.
TAE firstly calculates the pairwise cosine values in the dataset and perform the absolute value to map them into $[0, 1]$. According to Item 3, if we select a vector ${\bf u}_i$ randomly and suppose that it is the leading component, then the samples which correspond to a lower cosine value can be assumed as outliers. Following this route, for each sample ${\bf u}_i$ in the dataset ${\bf U}$, we compute the number of outliers ($\tau_i$) under an angular threshold $\eta_\theta$ by assuming ${\bf u}_i$ as the leading component. Thus, based on Item 1 and Item 2, those samples that contribute to the minimum $\tau_{\min} = \min \{\tau_i\}_{i=1}^n$ can be finally determined as outliers and be trimmed before the PCs are computed. We define the matrix ${\bf S}$ to select the remaining normal samples $\widetilde{\bf U} \in \mathbbm{R}^{D \times (n - \tau_{\min})}$, then the trimming can be formulated as $\widetilde{\bf U} = {\bf U} {\bf S}$. The detailed algorithm of TAE can be shown in Algorithm. \ref{algorithm-TAE}. 

\begin{algorithm}
    \caption{\textit{TAE Algorithm}}
    \label{algorithm-TAE}
    \begin{algorithmic} 
        \FOR{each sample ${\bf x}_i \in {\bf X}$}
        \STATE {${\bf u}_i \leftarrow {\bf x}_i / \|{\bf x}_i\|_2$}
        \ENDFOR
        \FOR{$i=1,\dots,n$}
        \STATE{$\tau_i = \text{Count}({\bf u}_i^T {\bf U} < \cos \eta_\theta)$}
        \ENDFOR
        \STATE{$\tau_{\min} \leftarrow \min \{\tau_i\}_{i=1}^n$}
        \STATE{Determine the selection matrix ${\bf S}$ based on $\tau_{\min}$.}
        \STATE{$\widetilde{\bf U} \leftarrow {\bf U} {\bf S}$ to determine the normal samples.}
        \STATE{${\bf Q} \leftarrow \arg \max\limits_{\bf Q} \text{trace} \left( {\bf Q}^T \widetilde {\bf U} \widetilde {\bf U}^T {\bf Q} \right)$}
    \end{algorithmic}
\end{algorithm}

The $n - \tau_{\min}$ normal samples $\widetilde{\bf U}$ is obtained after TAE trims the outliers from the original zero-mean dataset. Note that for a memory-allowed input, the pairwise cosine value matrix ${\bf C} = {\bf U}^T {\bf U}$ could be utilized to determine the outliers and be trimmed before decomposition. More specifically, for the row $i$ with $\tau_{\min}$, those column indexes $j$ with ${\bf C}_{ij} < \cos \eta_\theta$ correspond to outliers.

\subsection{Computational Acceleration and Complexity Analysis}

The computation of AE could be optimized when the dimensionality $D$ or sample size $n$ is large, which can be attributed to the cubed computational complexity for matrix decomposition. When only one of them is large, the decomposition can be computed effectively by weighing $D$ and $n$. Specifically, the matrix ${\bf U} {\bf U}^T$ with size $D \times D$ can be decomposed directly when $D < n$. When $n < D$, the decomposition can be conducted on the matrix ${\bf U}^T {\bf U}$ with size $n \times n$. When the dimensionality of low-dimensional subspace is low, the randomized algorithms \cite{martinsson2011randomized, szlam2014implementation} for the decomposition of matrix ${\bf U}$ could be used to raise the efficiency.

\begin{figure*}[!t]
    \centering
    \includegraphics[width=16cm]{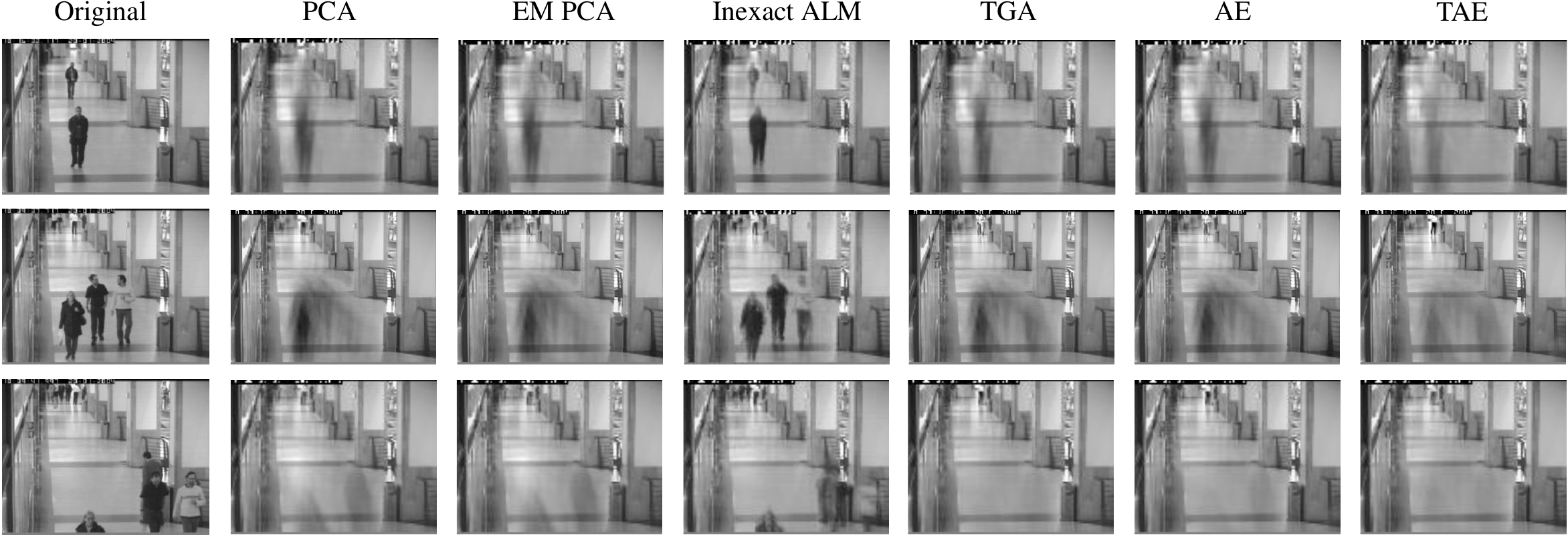}
    \caption{The reconstruction of background on \textit{CAVIAR} dataset. The frame length of video sequences are 1500, 1675 and 1675, respectively.}
    \label{CAVIAR-results}
\end{figure*}

\subsubsection{Computational complexity and memory requirements.} For a given zero-mean dataset ${\bf X} \in \mathbbm{R}^{D \times n}$, mapping the data to a $D-1$-sphere scales as $\mathcal{O}(D \times n)$ with normalizing. When $D < n$,computing the covariance matrix scales as $\mathcal{O}(n \times D^2)$ and decomposing a $D \times D$ matrix has computational complexity $\mathcal{O}(D^3)$. And when $D > n$, the decomposition can be transformed into matrix ${\bf U}^T {\bf U}$. Then the total computational complexity is $\mathcal{O}(\min(D^3, n^3))$ with $\mathcal{O}(\min(D^2, n^2))$ memory requirements. We use the full matrix decomposition when $D$ and $n$ is small enough, which scales as $\mathcal{O}(\min(D^3, n^3))$. In other cases, the randomized SVD is utilized with complexity $\mathcal{O}(kDnlog(n))$ when the top-$k$ PCs are needed.

For TAE, the trimming operation scales as $\mathcal{O}(n^2 \times D)$ and the remaining data matrix $\widetilde{\bf U}$ to be decomposed is of size $D \times (n - \tau_{\min})$, where $\tau_{\min}$ indicates the number of trimmed samples. When $\tau_{\min}$ is larger, the computation can be significantly reduced. 

\subsubsection{Large dataset in very high dimensions.} For large scale data, the memory usage can be optimized using randomized methods \cite{linderman2019fast} for PCA without the memory requirement for the entire data matrix at once. This improvement could produce nearly optimal accuracy by computing the decomposition efficiently. And the final computation speed of out-of-core algorithm relies on the disk access times.

\section{Experiments}

We develop AE to absorb both the advantages of PCA and RPCA based methods to provide a robust and flexible approach to fit different tasks. The goals of AE are summarized as

\begin{itemize}
    \item To simplify the description of high-dimensional data by extracting the most important properties with more appropriate subspaces.
    \item To obtain more robust \textit{principal components (PCs)} in data with (large scale) outliers than PCA, and to realize similar or better low-rank recovery performance like RPCA.
\end{itemize}

To verify the effectiveness of our proposed method, we conduct experiments on four different tasks, which cover the listed items above. For comparison, we use PCA, EM PCA \cite{roweis1998algorithms}, Inexact ALM \cite{candes2011robust}, and TGA \cite{hauberg2014grassmann}. All the experimental results are obtained by running on a personal computer with an Intel Core i7-8700 CPU and $2\times 8$ GB DDR4 2400MHz memory.

\subsection{Experimental settings}
The proposed AE is implemented with both Matlab and Python to provide comparisons in different experiments. 
For classification experiments, dimensionality reduction with AE and the compared methods is performed before the data is input into the classifier. For fear of the interference of classifier, all the classification tasks should be performed with XGBoost \cite{chen2016xgboost} and Random Vector Functional Link (RVFL) \cite{pao1994learning}, separately. The XGBoost classifier is provided by the Python Scikit-Learn wrapper interface\footnote{\url{https://xgboost.readthedocs.io/en/latest/python/python\_api.html\#module-xgboost.sklearn}}, and we utilize GridSearchCV\footnote{\url{https://scikit-learn.org/stable/modules/generated/sklearn.model_selection.GridSearchCV.html}} to conduct exhaustively search over specified parameter values to determine the important parameters. RVFL is implemented with Matlab, which is running by calling the MATLAB Engine API for Python\footnote{\url{https://www.mathworks.com/help/matlab/matlab-engine-for-python.html}}. In RVFL implementation, we use $N=50$ hidden neurons and \textit{Sigmoid} activation function.

For background modeling experiments, samples are trimmed for TAE with an angular threshold $\eta_\theta$ between $\pi / 3$ and $7 \pi / 18$, in radians. All the samples, which has a deviation angle greater than $\eta_\theta$, will be considered as outliers and trimmed. In the experiments of shadow removal, the angular threshold $\eta_\theta$ between $5 \pi /18$ and $\pi / 3$ are used.

\begin{figure*}[!t]
    \centering
    \includegraphics[width=16cm]{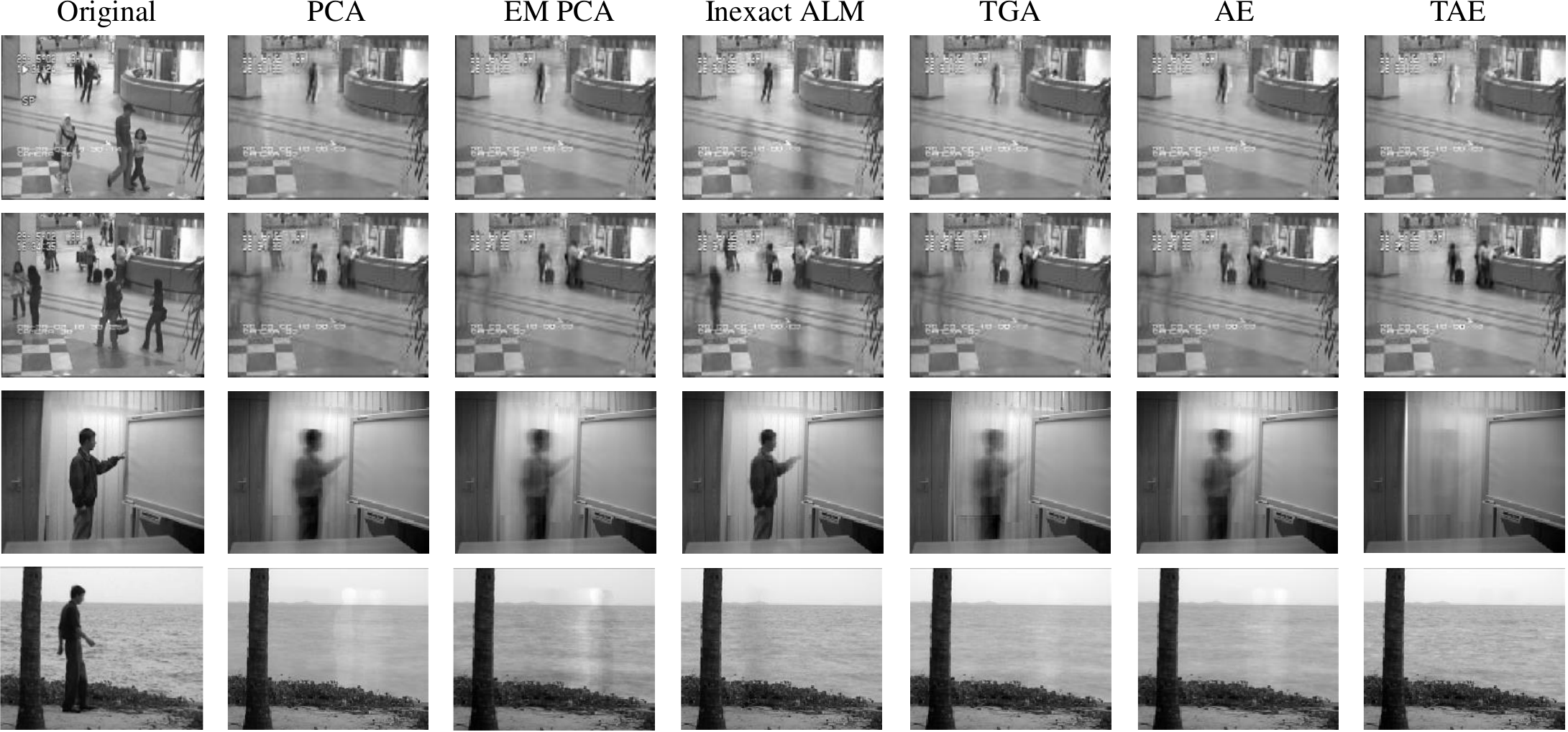}
    \caption{The reconstruction of background on \textit{I2R} dataset. The first two rows are scenes in \textit{Airport Hall} data. And the last two rows are corresponding to frames in \textit{Curtain} and \textit{Water Surface}, respectively.}
    \label{I2R-results}
\end{figure*}

\begin{table}[]
    \centering
    \begin{threeparttable}[b]
    \caption{Classification results ($\%$) on synthetic data.}
    \label{table-classification-synthetic-data}
    \begin{tabular}{cccccc}
        \toprule
        $d$ \tnote{a} &   Classifier  &   PCA     &   EM PCA  &   TGA     &   AE  \\
        \midrule
        \multirow{2}{*}{1}   &  RVFL   &   48.73  &   51.26  &   95.47  &   \textbf{96.03}  \\
        &  XGBC   &   51.34  &   53.41  &   \textbf{95.47}  &   \textbf{95.47}  \\
        \cmidrule{1-6}
        \multirow{2}{*}{2}   &  RVFL   &   99.76  &   99.76  &   99.76  &   \textbf{99.84}  \\
        &  XGBC   &   99.76  &   99.68  &   99.76  &   \textbf{99.84}  \\
        \bottomrule
    \end{tabular}
    \begin{tablenotes}
    \item[a] $d$ indicates the dimensionality of low-dimensional subspace.
    \end{tablenotes}
    \end{threeparttable}
\end{table}

For the comparative methods, the iteration threshold of EM PCA is set as 0.0001 as well as TGA, which is set a 50 percent trimming rate suggested by the author in \cite{hauberg2014grassmann}.

\subsection{Classification}

The classification experiments are conducted to prove the effectiveness of AE on subspace selection. AE and the compared methods are performed on both synthetic data and real data with outliers to determine PCs, and then the linear transformation with dimensionality reduction will be estimated through classification. In classification experiments, PCA, EM PCA and TGA are selected as competing methods.

\subsubsection{Synthetic data with vector-level outliers.} We generate synthetic data of 4 classes in 3-dimensional space and shift partial samples to add vector-level outliers intentionally. Dimensionality reduction with AE and the competing methods are conducted on synthetic data, the 2-dimensional results are shown in Fig. \ref{3d-data-dr}. The data in AE is normalized to reduce the error caused by vector-level translation so that the leading PC determined by AE is more reasonable. Thus, in the $1$-dimensional subspace shown in Table \ref{table-classification-synthetic-data}, AE exhibits noticeable advantage, in particular, compared with PCA and EM PCA. Accuracies of all the used methods are improved significantly, when the number of dimensionality orthogonal PCs increase from 1 to 2. However, AE maintains the highest classification accuracy in all compared methods.

\begin{table}[]
    \begin{threeparttable}[b]
    \centering 
    \caption{The mean classification accuracy ($\%$) on human action data.}
    \label{table-classification-real-data} 
    \begin{tabular}{cccccc}
        \toprule
        Data &   Classifier  &   PCA     &   EM PCA  &   TGA     &   AE  \\
        \midrule
        \multirow{2}{*}{\textit{FSD} \tnote{b}}   &  RVFL   &  84.66   &   84.89  &   84.35  &   \textbf{85.15}  \\
        &  XGBC   &   78.32  &   78.52  &   79.08  &   \textbf{79.41}  \\
        \cmidrule{1-6}
        \multirow{2}{*}{\textit{Kin} \tnote{c}}   &  RVFL   &   69.60  &   69.74  &   69.42  &   \textbf{70.02}  \\
        &  XGBC   &   61.19  &   61.68  &   61.84  &   \textbf{63.96}  \\
        \bottomrule
    \end{tabular}
    \begin{tablenotes}
        \item[b] \textit{FSD} indicates the \textit{FSD-10} dataset.
        \item[c] \textit{Kin} indicates the \textit{Kinematic} dataset.
    \end{tablenotes}
    \end{threeparttable}
\end{table}

\subsubsection{Human action data.} Outliers in human action data is hard to be addressed properly \cite{liu2020fsd, aviezer2012body}. We perform dimensionality reduction on human action data including \textit{FSD-10} \cite{liu2020fsd} and \textit{Kinematic}\footnote{\url{https://physionet.org/content/kinematic-actors-emotions/2.0.0/}} datasets, which covers temporal and spatial features simultaneously with complex outliers. For the needs of appropriate input size and efficiency improvements, feature representations are firstly encoded through a well-trained Spatial-Temporal Graph Convolutional Network (ST-GCN) \cite{yan2018spatial}. More specifically, the \textit{FSD-10} data has 510 features while the \textit{Kinematic} data is represented with 256 dimensions. Directional outliers, which can not be well addressed by a softmax classifier in ST-GCN, could lie on each possible dimensions. 

We compute the mean accuracy to avoid the error of randomly optimized classifier, when the input features are mapped to different dimensions. Actually, these extracted features with ST-GCN are mapped to $(0, 1)$, which significantly reduce the influence of outliers for $\ell_2$-norm based PCA. For AE, however, these outliers of lower value can be further suppressed under the squared cosine measurement. Thus, the performance of PCA based methods would be improved by AE. We compute the mean accuracy of classification on both \textit{FSD-10} and \textit{Kinematic} datasets from 10 to 100 dimensions (spaced by 10 dimensions), as shown in Table \ref{table-classification-real-data}. Compared with other comparative methods, AE always obtains the notable mean accuracy.

\begin{figure*}[!t]
    \centering
    \includegraphics[width=16cm]{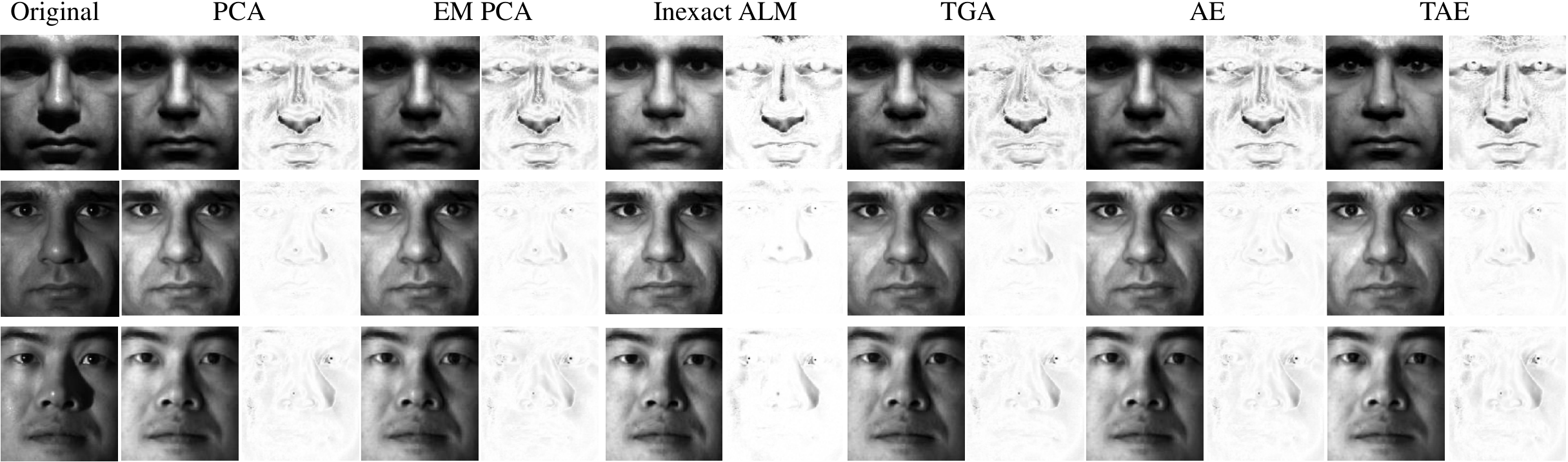}
    \caption{The results of shadow removal on \textit{Extended Yale Face Database B} with AE (TAE) and other compared methods. The faces are grouped by the original face, the reconstructed face and their inverted difference.}
    \label{YALE-results}
\end{figure*}

\subsection{Background Modeling}

RPCA based methods, which decompose a given data matrix into a low-rank matrix and a sparse one, extends the feasibility of PCA on background modeling tasks. The purpose of background modeling is to extract the static background from data such as video surveillance sequences. One most important character is that most frames containing pedestrians tend to be treated as outliers in the frame sequences and to become notable features. Moreover, PCA based methods are sensitive to outliers so that they usually lead to unsatisfactory recovery performance such as PCA and EM PCA. For the purpose of addressing data with large scale outliers, TAE is conducted in the background modeling experiments to identify and trim outliers preliminarily under cosine measurement. The experiments demonstrate that TAE provide a more robust approach on background modeling tasks based on AE. 

Firstly, we test the performance of background reconstruction on \textit{CAVIAR}\footnote{\url{http://homepages.inf.ed.ac.uk/rbf/CAVIAR/}} video sequences, which are also used in \cite{hauberg2014grassmann}. The selected 3 frames are shown in Fig. \ref{CAVIAR-results}. The backgrounds reconstructed by AE without trimming produce less ghosting than PCA and EM PCA, while the performance is also limited by massively scattered outliers. When the angular threshold of trimming decreases (enable trimming), TAE trims outliers appropriately and updates PCs so that the ghosting obtained by TAE reduces significantly. 

We also repeat the experiments on \textit{I2R} dataset \cite{li2004statistical} used in RPCA and TGA. For PCA, EM PCA, TGA and AE, we use 5 principal components to reconstruct the backgrounds. Fig. \ref{I2R-results} shows the extracted backgrounds from 6 frames. Compared with other methods, the backgrounds can be reconstructed by TAE with less ghosting. 

\subsection{Shadow Removal on Face Data}

Shadow removal is another commonly used estimation for RPCA based methods. We conduct experiments on the face data to test the effectiveness of AE/TAE on face modeling applications. The \textit{Extended Yale Face B Database} \cite{georghiades2001from}, which consists of 64 faces with size $192 \times 168$ under different illumination for each subject, can be utilized to remove the cast shadows in a low-dimensional subspace. These pixel-wise differences of cast shadows are outliers for PCA based methods. 

We reconstruct the tested faces with 9 principal components and show the inverted absolute difference between the original faces and the reconstructed ones in Fig. \ref{YALE-results}. Inexact ALM removes both the cast shadows and saturations so that the reconstructed faces appears more matte. Yet PCA based methods seems to preserve more illumination and remove the most cast shadows compared with the original face. Compared with AE, TAE with trimming reduce the influence of outliers and reconstruct more clear faces.

\subsection{Quantified Running Time}

\begin{table}[]
    \centering 
    \caption{The mean running time (seconds) with AE/TAE and the comparative methods.}
    \label{table-running-time} 
    \begin{tabular}{ccccccc}
        \toprule
        Data & PCA & EM & ALM & TGA & AE & TAE\\
        \midrule 
        \textit{Yale} & 0.3 & 13.8 & 43.1 & 19.1 & \textbf{0.1} & 0.2 \\
        \textit{Water} & 0.6 & 3.8 & 271.2 & 15.2 & 0.5 & \textbf{0.4} \\
        \textit{Hall} & 3.5 & 9.4 & 4611.4 & 89.8 & \textbf{3.3} & 3.6 \\
        \textit{CAVIAR} & 7.4 & 17.3 & 5520.6 & 55.1 & 7.2 & \textbf{5.9} \\
        \bottomrule
    \end{tabular}
\end{table}

We compute the mean running time in the background modeling and shadow removal experiments to test the scalability of AE/TAE, as shown in Table \ref{table-running-time}. The utilized datasets include \textit{Extended Yale Face Database B} (size of $32256 \times 64$), \textit{I2R Water Surface} (size of $20480 \times 633$), \textit{I2R Airport Hall} (size of $25344 \times 3584$), and \textit{CAVIAR} (size of $110592 \times 1675$). Benefiting from the randomized algorithm for matrix decomposition, PCA, AE and TAE are more efficient when less singular vectors are required. The computational complexity in TAE could be further reduced by the pre-trimming for outliers, when the sample size $n$ is high.

\section{Conclusion and Future Work}

We propose a new robust angular density based approach (AE) for PCA with cosine measurement. Based on the angular embedding framework, we prove that squared cosine measurement achieves angular attention and takes effect on outlier suppression. For large-scale or high-dimensional data, the computation is improved by weighing the computational complexity of the dimensionality and sample size. Furthermore, to address data with large scale outliers, the trimmed AE (TAE) approach are proposed in this paper to provide a pre-trimming mechanism. Through extensive experiments, the proposed AE/TAE outperforms the competing methods. 

In the future, the proposed AE/TAE can be conducted subsequently. Firstly, it is possible to reformulate an iterative approach under AE framework without matrix decomposition like EM PCA \cite{roweis1998algorithms}. Secondly, we hope the pre-trimming theory introduced in TAE can be extended to other algorithms or applications. 


\bibliography{bibfile}
\end{document}